\documentclass[letterpaper, 10pt, conference]{ieeeconf}  
\pdfoutput=1 

\IEEEoverridecommandlockouts                              
\let\labelindent\relax

\usepackage{graphicx}
\usepackage{url}            
\usepackage{amsfonts}       
\usepackage{nicefrac}       
\usepackage{color}
\usepackage{dsfont}
\usepackage{soul}
\usepackage{bbm}
\usepackage{enumitem}
\usepackage{placeins}
\usepackage{dblfloatfix}
\usepackage{amsmath}
\usepackage{amssymb}
\usepackage{subfig}
\usepackage[noend]{algpseudocode}
\usepackage{algorithm,algorithmicx}

\setlist[itemize]{leftmargin=*}
\setlist[enumerate]{leftmargin=*}

\newtheorem{thm}{Theorem}
\newtheorem{lemma}{Lemma}
\newtheorem{cor}{Corollary}
\newtheorem{defn}{Definition}
\newtheorem{rem}{Remark}

\newcommand{\Sym}{\text{Sym}_n(\mathbb{R}_{\ge 0})}

\title{Unsupervised Metric Learning In Presence of Missing Data}

\author{
  Anna C.~Gilbert \\
  Department of Mathematics\\
  University of Michigan - Ann Arbor\\
  Ann Arbor, MI 48104 \\
  \texttt{annacg@umich.edu} \\
  \and
  Rishi Sonthalia \\
  Department of Mathematics\\
  University of Michigan - Ann Arbor\\
  Ann Arbor, MI 48104 \\
  \texttt{rsonthal@umich.edu}\\
}

\begin{document}

\maketitle
\thispagestyle{empty}
\pagestyle{empty}

\begin{abstract}

For many machine learning tasks, the input data lie on a low-dimensional manifold embedded in a high-dimensional space and, because of this high-dimensional structure, most algorithms inefficient. The typical solution is to reduce the dimension of the input data using a standard dimension reduction algorithms such as {\sc Isomap, Laplacian Eigenmaps} or {\sc LLEs}. This approach, however, does not always work in practice as these algorithms require that we have somewhat ideal data. Unfortunately, most data sets either have missing entries or unacceptably noisy values. That is, real data are far from ideal and we cannot use these algorithms directly.

In this paper, we focus on the case when we have missing data. Some techniques, such as matrix completion, can be used to fill in missing data but these methods do not capture the non-linear structure of the manifold. Here, we present a new algorithm {\sc MR-Missing} that extends these previous algorithms and can be used to compute low dimensional representation on data sets with missing entries. We demonstrate the effectiveness of our algorithm by running three different experiments. We visually verify the effectiveness of our algorithm on synthetic manifolds, we numerically compare our projections against those computed by first filling in data using nlPCA and mDRUR on the MNIST data set, and we also show that we can do classification on MNIST with missing data. We also provide a theoretical guarantee for {\sc MR-Missing} under some simplifying assumptions.
\end{abstract}

\section{Introduction}


Many real world data sets can be reasonably modeled as low dimensional manifolds embedded in much higher dimensional spaces, and for these models, a class of techniques play a crucial role in revealing or learning these intrinsic manifolds. This class includes {\sc Isomap}~\cite{Tenenbaum2000}, Local Linear Embedding ({\sc LLE})~\cite{Roweis2000}, {\sc Hessian-LLE}~\cite{Donoho2003}, {\sc Maximum Variance Unfolding}~\cite{Weinberger2004}, {\sc KNN-Diffusion}~\cite{Coifman2006}, and {\sc Laplacian Eigenmap}~\cite{Belkin2003}. All of these algorithms have the basic structure shown in Figure~\ref{fig:manifold_learning_structure}.  Given data, we compute a distance matrix (alternatively, we are given a distance matrix), from which we determine neighborhoods about each data point. Some of these algorithms find the $K$ closest points to each data point and others determine the data points in an $\epsilon$ neighborhood about each data point. Each algorithm then uses this neighborhood information to compute local Euclidean coordinates in some fashion (i.e., to learn the underlying manifold) and, thus, to determine a low-dimensional representation for the data.

\begin{figure}[h!]
\includegraphics[width = \linewidth]{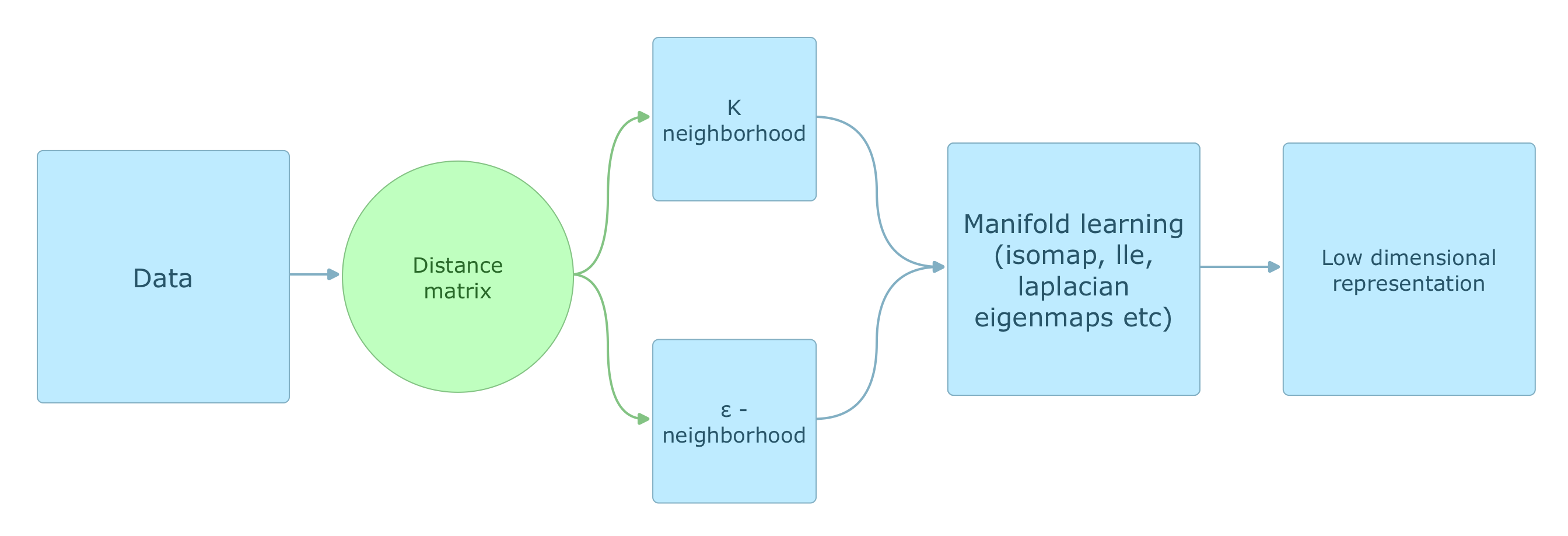}
\caption{General Manifold learning procedure}
\label{fig:manifold_learning_structure}
\end{figure}

Each of these algorithms assumes that the given or computed distance matrix adheres to a metric. To see the importance of using a distance matrix that satisfies a metric with these algorithms, we show in Figure~\ref{fig:swissroll} the impact upon the manifold returned by {\sc Isomap} when we corrupt or perturb the distance matrix so that it no longer satisfies a metric. Note that these algorithms are robust to (small) perturbations in the \emph{data} but not in the \emph{distances} among the data points. In the top left figure (a), we have the original swissroll dataset with 2,000 points. It is a two dimensional manifold embedded in three dimensions. When we run {\sc Isomap} on the true distance matrix, we see in the upper right figure (b), an ``unrolled'' version of the intrinsic manifold. To corrupt the distance matrix, we add i.i.d.~Gaussian noise $\sim \mathcal{N}(0,0.01)$ to each non diagonal entry of the distance matrix. We then replace all negative entries with zero and preserve symmetry by averaging the corrupted distance matrix with its transpose. The perturbed distances may not satisfy the triangle inequality and, hence, the corrupted distances may not adhere to a metric. 

The lower left figure (c) is the embedding from the corrupted distance matrix. It is considerably different from the original embedding, points are missing, and there is no apparent lower dimensional manifold structure at all. Because the distances do not satisfy a metric, the points all collapse to one location and cannot be distinguished in the figure.  In the lower right figure (d), we first repair the corrupted distance matrix, using a metric repair algorithm from~\cite{Gilbert2017}, and then embed the data using {\sc Isomap}. The resulting embedding is much closer to the original embedding, with some minor distortion. Thus, we can see that unless the distance matrix satisfies a metric, dimension reduction or manifold embedding algorithms fail catastrophically.

\begin{figure}[h!]
\centering
\subfloat[Original swissroll data set]{\includegraphics[width=0.4\linewidth]{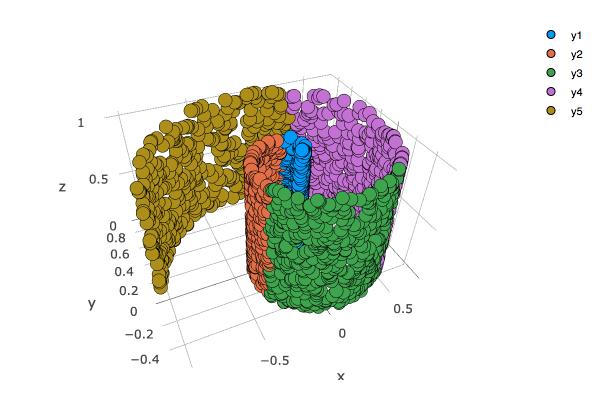}}
\subfloat[true distance matrix]{\includegraphics[width = 0.4\linewidth]{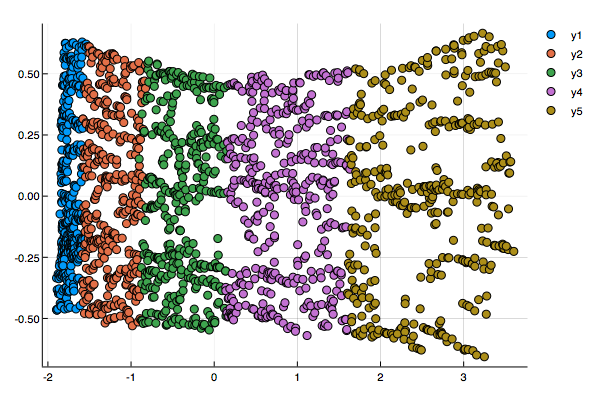}} \\
\subfloat[corrupted distance matrix]{\includegraphics[width = 0.4\linewidth]{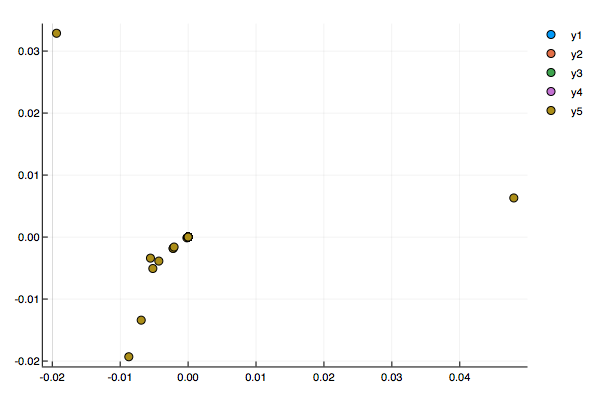}}
\subfloat[repaired distance matrix]{\includegraphics[width = 0.4\linewidth]{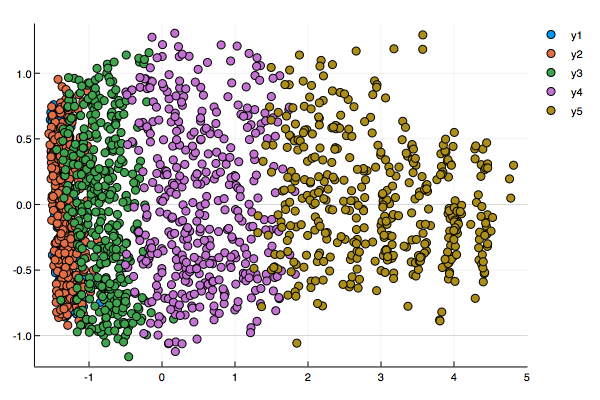}}
\caption{(a) The original swissroll data set (2000 points) and the results from {\sc Isomap} for: (b) the original distance matrix, (c) the corrupted distance matrix, and (d) the repaired distance matrix.}
\label{fig:swissroll}
\end{figure}

\subsection{Problem Set Up}

This example illustrates the main problem we address: if we have either missing data or missing or corrupted entries in the distance matrix and we assume that the data come from an intrinsic low dimensional manifold, compute a low dimensional representation of the imputed or corrected data set.   One such approach for missing data is to complete the data matrix using a matrix completion algorithm. These algorithms assume that the data matrix is approximately low rank and fills in the missing entries accordingly. These algorithms fit the data to a {\em linear} subspace rather than an intrinsically {\em nonlinear} embedding and may miss key features of the data. Other methods that learn the intrinsic low dimensional structure in a data set impute missing data values in the original space but one cannot use other, potentially better, algorithms for the embeddings. Our method aims to repair the distance matrix of the data, so as to extend existing embedding algorithms such as {\sc Isomap} and {\sc Laplacian Eigenmaps} to handle missing data or corrupted distances.

To be precise, let $X$ be the high-dimensional data set and $D$ the distance or dissimilarity matrix amongst the data points, and consider the following four problem scenarios:
\begin{enumerate} 
\item The data set $X$ has corrupted entries, 
\item The data set $X$ has missing entries,
\item The dissimilarity matrix $D$ has corrupted entries, or
\item The dissimilarity matrix $D$ has missing entries. 
\end{enumerate}

In the first model $X$ has added noise and many of the traditional algorithms are robust and produce satisfactory results. Hence, we shall focus on the second scenario where $X$ has missing entries. The second scenario covers the last two (as missing data corrupt the distances between points) and we sketch the applications of our methods to these scenarios and leave a more in depth analysis of those models as future work.

\subsection{Previous work}

There are two main approaches to filling in missing manifold data that we summarize below. Both methods strive to learn, in an unsupervised fashion, a representation of the data and then to use that learned representation to fill in the missing values. The first method employs a low-dimensional representation as an intermediary step in the overall data representation while the second method directly learns a low-dimensional representation. We will also discuss how matrix completion algorithms could be used for model 4\\

\noindent\textbf{Non-linear Principle Component Analysis (nlPCA)~\cite{nlPCA}.} This method is a non-linear analog to principle component analysis. The idea is to use a five-layer neural network with architectural dimensions $n \times m \times d \times m \times n$, where $n$ is the input dimension, $m$ is usually bigger than $n$, and $d$ is the desired dimension of the embedding. We train the neural network to learn the identity map so that the middle layer with $d$ neurons is the low-dimensional representation of the data. To extend to missing data, the network is also trained to reproduce the input data but during the training procedure, any gradients that depend on missing values are disregarded. Then, to fill in the missing data, the data with missing values is input to the network and the output values are used to fill in any missing data. \\

\noindent\textbf{Missing Data Recovery through Unsupervised Regression (mDRUR)~\cite{mDRUR}.} The second major method is the missing data recovery through unsupervised learning (mDRUR). Similarly to the nlPCA algorithm, this algorithm is an extension of a dimensionality reduction algorithm known as dimensionality reduction through unsupervised learning (DRUR). Using the notation of~\cite{mDRUR}, let $Y$ be the high dimensional representation of the data and $X$ the low dimensional representation. Then we have two maps $f,F$ such that $Y = f(X)$ and $X = F(Y)$. To learn the low-dimensional representation, we minimize 
\[ 
	\arg \min_{X,f,F} \|Y - f(X)\|_F^2 + \|X - F(Y)\|_F^2 + \lambda_f R(f) + \lambda_FR(F) 
\] 
where $R(f)$ and $R(F)$ are regularization terms. To fill in missing data, we first use a linear matrix completion method to fill in the missing values, then we use a spectral method to compute $X$. Finally, we learn the low-dimensional representation as above and optimize over the missing values of $Y$ to fill in the missing data. 

Both algorithms are primarily dimensionality reduction algorithms and, as such, we must first fix the lower dimension and then solve an optimization problem. In most real world applications we do not know the optimal low dimension and, hence, in order to fill in the missing data, we must first find this dimension and then fill in the data, rather than separating these two tasks. Furthermore, the imputed values depend on the computed, specific reduced representation; we cannot avail ourselves of a variety of dimension reduction algorithms and obtain what we hope to be a consistent or robust approximation of the missing values.  \\

\noindent\textbf{Euclidean distance matrix and metric completion.}
One might be tempted to restrict our distances to Euclidean distances as it is well known~\cite{Gower1985} that Euclidean distance matrices (with squared Euclidean distance entries) are low rank matrices with rank at most $r+2$ if $r$ is the dimension of the space in which the points lie. Hence, the problem of Euclidean matrix completion can be solved using standard low rank matrix completion algorithms. Additionally, this specific problem has been further studied with many successful algorithms in~\cite{EDC1,EDC2,EDC3}.

In some cases we want our original data to follow a non-Euclidean metric. For example, it is has been shown that for the MNIST dataset if we use the tangent distance metric instead of Euclidean, then $k$ nearest neighbor classifiers have better performance. In this case, if we have the local neighborhood information for the data we can still run {\sc Isomap} and other various algorithms to get lower dimensional representations.  Gower~\cite{Gower1985} showed that these matrices are either low rank (with the same low rank condition as before) or have full rank. In the case that they have full rank, we can no longer use matrix completion algorithms. 

Even when we have a low rank matrix, we only know we can complete these matrices with high probability if the entries present are sampled according to a certain distribution. Finally, even if we can successfully apply these matrix completion algorithms, there are no guarantees that the resulting distance matrices satisfy a metric.

\subsection{Our approach and contributions}

We focus on the second scenario. We separate the problem into three steps. First we estimate distances between the data points, we then correct these distances so that they adhere to a metric, finally we run a suitable dimension reduction algorithm. 
We use the Increase Only Metric Repair (IOMR) algorithm in Gilbert and Jain~\cite{Gilbert2017} to repair the inaccurate distance matrices. As we can see in Figure~\ref{fig:manifold_learning_structure}, all of the dimension reduction algorithms depend on the local distances and not on the actual data points themselves. Hence, filling in the missing data is both costly and unnecessary. Instead, we first estimate the distance matrix from the incomplete data, then we correct it. This approach has two advantages over the previous methods

\begin{itemize} 
\item  No parameters: Our algorithm has no parameters that need tuning. Hence making it faster and easier to train compared to nlPCA and mDRUR. Additionally, our algorithm is quadratic is the number of dimensions. Hence, its performance scales well with number of dimensions. 
\item Accuracy: The manner in which we estimate the distances and then correct them is geared to exactly preserve the local structure. 
\end{itemize}

\begin{rem} For the most general version of scenario three, fast approximation algorithms for the general metric repair problem do not currently exists. A class of approximation algorithms can be found in Gilbert and Sonthalia~\cite{GilbertSonthalia:GraphMetric2018}.
\end{rem}

The rest of the paper is organized as follows, Section 2 presents background knowledge, Section 3 presents our algorithms for correcting a corrupted distance matrix so as to produce an accurate low dimensional embedding of a data set. We focus specifically on missing data. Section 4 provides our experimental results. 
\section{Background}

\subsection{Manifolds and Geodesic distances}

\begin{defn} \label{defn:mani} $M \subset \mathbb{R}^n$ is called a $d$ dimensional manifold is for all $x \in M$ there exists an $\epsilon > 0$ such that there if a continuous bijective function $f$ from $N = \{ y \in M : \|x-y\| < \epsilon\}$ to $\mathbb{R}^d$ such that the inverse is continuous as well. 
\end{defn}

Intuitively the above definition says that if we look at a $d$-dimensional manifold $M$ and if we zoom in close enough to any point then it looks like we are in $\mathbb{R}^d$. For example the swiss roll (from the introduction) is a two-dimensional manifold because near any point it looks like a plane. As we can see from the definition itself the local structure of a manifold is important. Hence, all dimensionality reduction algorithms start by computing the local neighborhood of each point. This is done in one of two ways:
\begin{itemize}
\item determine the $k$ nearest neighbors, or
\item compute the neighbors within some distance $\epsilon$
\end{itemize}

These local neighborhoods then overlap to describe the general manifold structure. Hence, having this correct local structure is crucial for the success of any of the dimensionality reduction algorithms. 

In this paper we will we focus on using {\sc Isomap} (though we could have picked any of the other algorithms). For {\sc Isomap} once we have the graph (i.e., two data points are adjacent if one is in the local neighborhood of the other) we then compute the shortest distance (along this graph) between all the points. This usually done using the Floyd Warshall algorithm. We store these distances in a matrix $\tilde D$. These distances are known as geodesic distances. They are the distances we want between our data points in the low dimensional representation. 

\subsection{Multidimensional Scaling}

The technique used to go from the inferred distance matrix $\tilde D$  to an actual embedding is called multidimensional scaling (MDS) and we include this discussion to complete Section~\ref{sec:experiments}. An important point to note is that the points we recover are not unique, since rotating and translating the embedding will not change the pairwise distances between the points. 

Suppose we have $n$ by $n$ distance matrix $\tilde D$. We want to find points $x_1, \hdots, x_n$ in some $d$ dimensional Euclidean space such that ${\rm dist}(x_i,x_j) = \tilde D_{ij} = \tilde D_{ji}$. Let us define 
\[ 
	S = -\frac{1}{2}\left(I - \frac{1_n 1_n^T}{n}\right) \tilde D \circ \tilde D 
	     \left(I - \frac{1_n 1_n^T}{n}\right) 
\] 
where $1_n$ is the $n$ dimensional vector of all 1s and $\circ$ is the Hadamard product, which we multiply the two matrices coordinate wise.. Then using the fact that the points are translation invariant (we can assume the centroid of $x_1, \hdots, x_n$ is the origin) using which it can be shown that $S_{ij}$ would then be $x_i \cdot x_j$. This matrix is now positive semi-definite, hence has an eigenvalue decomposition \[ S = U \Lambda U^T \] Where $\Lambda$ is a diagonal matrix of the eigenvalues. If we then define $X = U \Lambda^{0.5}$. Then this is our embedding. We get a $d$ dimensional embedding by using only the biggest $d$ eigenvalues. 

\subsection{Metric Repair}

Gilbert and Jain in \cite{Gilbert2017} defined the \emph{sparse metric repair} problem. They define $\Sym$ to be the set of positive real symmetric matrices. More generally, let us define Sym$_n(S)$ to be the set of symmetric matrices with entries drawn from $S$. Then, for any matrix $D \in \Sym$ we say it satisfies a metric if the diagonal of $D$ is all 0s and for all $i,j,k$ we have that $D_{ij} \le D_{ik} + D_{kj}$. 

The sparse metric repair problem seeks a solution to the following optimization problem: Given $D \in \Sym$ and $S \subset \mathbb{R}$ 
\begin{equation} 
\label{eq:sparse_opt}
    {\rm argmin} \|P\|_0 \,\, \text{s.t. } D+P \text{ is a metric and } P \in \text{Sym}_n(S), 
 \end{equation} 
 where $\|\cdot \|_p$ is the vector $\ell_0$ is the pseudonorm that counts the number of non-zero entries. In \cite{Gilbert2017}, they define three variants of the corresponding to three different: $S = \mathbb{R}_{\le 0}$ decrease only metric repair ({\sc Domr}), $S = \mathbb{R}_{\ge 0}$ increase only metric repair ({\sc Iomr}), and $S = \mathbb{R}$, which is simply called metric repair or MR. In general, for any given set $S$ we shall refer to problem as MR($S$)

For the rest of paper we will focus on the increase only case and use the following algorithm from \cite{Gilbert2017}:

\begin{algorithm}
\caption{IOMR Fixed}
\begin{algorithmic}[1]
\Require{$D \in \Sym$ }
\Function{IOMR-Fixed}{D}
\State $\hat{D} = D$
\For{$k \gets 1 \textrm{ to } n$}
\For{$i \gets 1 \textrm{ to } n$}
\State $\hat{D}_{ik} = \max(\hat{D}_{ik}, \max_{j < i}(\hat{D}_{ij}-\hat{D}_{jk}))$
\EndFor
\EndFor
\State \Return{$\hat{D}-D$}
\EndFunction
\end{algorithmic}
\end{algorithm}

While Gilbert and Jain showed that {\sc IOMR-Fixed} worked empirically, they could not provide any guarantees on its performance nor did they demonstrate how it could be used for actual applications. Gilbert and Sonthalia \cite{GilbertSonthalia:GraphMetric2018} provides a more in-depth analysis of a generalized problem. One of their results is to show that the problem of increase only metric repair is NP-Hard for even simple sets $S$. They also provide several approximation algorithms for this problem and more general variants.

As it was noted in the \cite{Gilbert2017}, one could relax Equation~\eqref{eq:sparse_opt} to a convex optimization problem by minimizing $\|P\|_1$ instead of $\|P\|_0$. Gilbert and Jain used various different convex optimization methods, and while these methods produce satisfactory results, they were extremely slow. Hence, we stick with {\sc IOMR-Fixed}. The speed of the optimization algorithms will not scale well with the number of data points as that the output solution $D'=D+P$ needs to be a metric. In particular, for any $i,j,k$ we need the entries for $D'$ to satisfy the triangle inequality. Thus, if we have $n$ data points, we have $O(n^3)$ constraints. Even for simple applications where we have 1000 data points, the optimization problem has $O(10^9)$ constraints. 

\captionsetup[table]{skip=10pt}

\section{Metric repair on manifolds}

In our model scenario, we are given an incomplete data set $X$ and a matrix $Q$ that specifies the support of the \emph{known} entries. We present an algorithm to compute a (potentially corrupted) distance matrix from the incomplete data and then use metric repair to correct the perturbed distances. The idea is to ignore missing entries when calculating the distances between two points as in \cite{Balzano2010}.  Thus, we have the following algorithm. 

\begin{algorithm}
\caption{MR-missing}
\begin{algorithmic}[1]
\Require{$X$ Input data, $Q$ support of the data ($Q_{ij} = 1$ if and only if $X_{ij}$ is present. $Q_{ij} = 0$ otherwise)}
\Function{MR-Missing}{X,Q}
\State $D = zeros(n,n)$
\For{$i \gets 1 \textrm{ to } n$}
\For{$j \gets 1 \textrm{ to } n$}
\State $D_{ij} = \left(\sum_{k=1}^n Q_{ik}\cdot Q_{jk} (X_{ik} - X_{jk})^2\right)^{1/2}$
\EndFor
\EndFor
\State P = IOMR-Fixed(D)
\State {\sc Isomap}$(D+P)$
\EndFunction
\end{algorithmic}
\end{algorithm}

As we will see in Section~\ref{sec:experiments}, this algorithm works well in practice and, in this section, we provide theoretical analysis of its performance in a model setting. Before doing so, let us develop some intuition about the algorithm. To analyze how well the algorithm performs, supposed we had a probabilistic model with the following two assumptions:
\begin{enumerate}
\item if two data points are far apart, then our estimated distance is small with low probability; and,
\item if two data points are far and our estimated distance is small, then, with high probability, we increase this distance during metric repair.
\end{enumerate} 
Assuming the above two assumptions hold, we argue that {\sc MR-Missing} maintains the local structure of the data. When we compute the distance between two data points so as to ignore missing entries, the estimated distance is smaller than the true distance. Thus, if two points are initially close, then they remain close together. This is beneficial since the crucial structure for all manifold learning algorithms is the local distances. It could happen, that because data are missing, two points $x,y$ that were not close together initially have small inferred distance. Then, by assumption (1), this would happen with small probability and, by assumption (2), we would fix this distance with high probability. Finally, since the metric has been repaired and we have, with high probability preserved the local structure of the data set, {\sc Isomap} produces an embedding consistent with that of the full data set. That is, with high probability, the algorithm preserves the local structures and guarantees that all the distances adhere to a metric, and we conclude the low dimensional embeddings calculated with this repaired distance matrix to preserve most of the manifold structure of the data.

\subsection{Theory Result}
There are two steps to our algorithm. The first step is estimating distances between points with missing coordinates. The second is increasing these distances so that the distances adhere to a metric. In this subsection, we are going to analyze the effectiveness of {\sc MR-Missing} by showing that in the following model if two points are well separated then the distance estimated by step 1 is large with high probability. 

In this model our data consists of two Gaussians clusters in $\mathbb{R}^n$ with means $\mu_1,\mu_2 \in \mathbb{R}$ (all coordinates for a data point from one cluster have the same mean) and covariances $\Sigma_1 = \Sigma_2 = 0.5I_n$, where $I_n$ is the $n$ dimensional identity matrix. 

We assume that, for any data point $x$, each of its $n$ coordinates is present independently with probability $p$. Under these conditions, we want to show that with high probability our algorithm preserves local neighborhoods. 

For notational connivence let us define $d_p(x,y)$ to be the distance between $x,y$ estimated by step 1 of {\sc MR-Missing} when each coordinate is present independently with probability $p$. 

Before we can state and prove our result we need a few lemmas first

$ $

\begin{lemma}  [Birg\'e 2001, \cite{birge}] \label{lem:Birge} 
For all $D \ge 1$, if $X = Z_1^2 + \hdots + Z_D^2$, where $Z_i \sim \mathcal{N}(\mu_i,1)$, and $\displaystyle \lambda = \sum_{i=1}^D \mu_i^2$, then, all $0 < c < D+\lambda$, we have that 
\[ 
	\Pr[X \le c] \le e^{-\frac{(D+\lambda-c)^2}{4(D+2\lambda)}}.
\] 
\end{lemma} 

\begin{lemma} [Hoeffding's Inequality] \label{lem:hoeffding} 
If we have $n$ i.i.d. variables $X_1, \hdots, X_n$ such that $X_i = 1$ with probability $p$ and $X_i = 0$ with probability $1-p$, then, for all $\epsilon > 0$, we have that 
\[ 
	\Pr\left[ \left|\sum_{i=1}^n X_i - pn\right| \le \gamma n \right] \ge1-2 e^{-2\gamma^2n}.
\] 
\end{lemma} 
Lemma~\ref{lem:Birge} allows us to bound the tail of a non-central $\chi$-squared distribution and Lemma~\ref{lem:hoeffding} allows us to bound the probability that we have too much data missing. Combining these two Lemmas, we have the following theorem.
\begin{thm} 
Suppose $X \sim \mathcal{N}(\mu_1 \boldsymbol{1}, 0.5I)$ and $Y \sim \mathcal{N}(\mu_2 \boldsymbol{1}, 0.5I)$ are two points in $\mathbb{R}^n$ such that each coordinate of $X,Y$ is missing with probability $p$. If $\mu = \mu_1 - \mu_2$ and $q = p^2$, then, for all $q(1+\mu^2) > \epsilon > 0$ and  $ \frac{q(1+\mu^2) - \epsilon}{(1+\mu^2} > \gamma > 0$, we have that 
\[ 
	\Pr[d_p(x,y) < \epsilon n] \le e^{-2\gamma^2n} + 
			\left(e^{-\frac{((q-\gamma)(1+\mu^2)-\epsilon)^2}{4(q-\gamma)(1+2\mu^2)} } \right)^{n}. 
\] 
\end{thm}

\begin{proof} Let $Z = X-Y$. Then $Z_1, \hdots, Z_n$ are i.i.d Gaussian random variables with mean $\mu = \mu_1 - \mu_2$ and variance 1. We know from {\sc MR-Missing} that we use the entry $Z_i$ to calculate the distance between $X,Y$ if and only if both $X_i$ and $Y_i$ are present. This happens with probability $p^2$, which we define as $q = p^2$. Thus, we have the entry $Z_i$ with probability $q$. 

Let $q(1+\mu^2) > \epsilon > 0$. We want to show that the probability that the distance calculated by step 1 of {\sc MR-Missing} is greater than $\epsilon$ is small. To do this, we use Hoeffding's inequality to divide into two cases, one in which we observe a large number of coordinates and one a few coordinates. We shall see that the case when we observe a few coordinates occurs with low probability and we obtain a bound on the distance. We will then see that for the case where we see a large number of coordinates, then the distance is large with high probability. 

Let $\displaystyle 0 < \gamma < \frac{q(1+\mu^2) - \epsilon}{(1+\mu^2)}$. Then, by definition of $\epsilon$, we see that $q > \gamma > 0$. Let $K$ be the number of entries we observe and, by Lemma \ref{lem:hoeffding}, we have that 
\[ 
	\Pr[ K \le (q-\gamma)n] \le e^{-2\gamma^2n}.
\] 
For notational convenience, let $D = (q-\gamma)n$ and consider two cases.

\textbf{Case 1:} Suppose $K \le D$. That is, we observe fewer than $D$ entries. Then, by Hoeffding's Inequality, we know that this happens with probability at most $ e^{-2\gamma^2n}$. 

Thus we have that \[ \Pr[d_p(x,y) < \epsilon n | K \le D] \Pr[K \le D] \le e^{-2\gamma^2n} \]

\textbf{Case 2:} Suppose $K \ge D$. That is, we observe a large number of entries. Let us condition on the actual value of $K$. Suppose that $K=k$ and we have observed entries $Z_{i_1}, \hdots, Z_{i_k}$. We know these entries are i.i.d. with mean $\mu$ and we want to bound 
\[ 
	p_k := \Pr\left[ \sum_{j=1}^k Z_{i_j}^2 \le \epsilon n  \right].
\]

In this case we see that the overall probability that we have a small distance, given that $K \ge D$, is  
\[ 
	\sum_{k= D}^n \Pr[K=k] p_k. 
\] 
The next thing to observe is that $p_k$ is monotone decreasing in $k$ because each $Z_i$ is non-negative with non-zero mean. Thus, we have  the following upper bound
\begin{align*}  
	\Pr[d_p(x,y) < \epsilon n | K \ge D]  &\le \sum_{k= D}^n \Pr[K=k] p_k \\
										  &\le  \sum_{k= D}^n \Pr[K=k] p_{D} \\
										  &\le p_{D}. 
\end{align*}

Now, we can use our tail bound for the $\chi$-squared distribution with $\displaystyle c = \epsilon n$, $D =(q-\gamma)n$, and $\lambda = D\mu^2$.  Thus,  we have 
\begin{align*} 
	p_{D} &\le e^{-\frac{((q-\gamma)n(1+\mu^2)-\epsilon n)^2}{4(q-\gamma)n(1+2\mu^2)} }\\
	      &= \left(e^{-\frac{((q-\gamma)(1+\mu^2)-\epsilon)^2}{4(q-\gamma)(1+2\mu^2)} } \right)^{n}.
\end{align*} 
 
Combining both cases, we see that the probability that the distance calculated is less than $\epsilon n$ is at most 
\[ 
	\Pr[d_p(x,y) < \epsilon n] \le e^{-2\gamma^2n} + \left(e^{-\frac{((q-\gamma)(1+\mu^2)-\epsilon)^2}{4(q-\gamma)(1+2\mu^2)} } \right)^{n}.
\]  
\end{proof}

Let us take a closer look at the effect of the various parameters on the above probability: 
\begin{itemize}
\item \emph{The dimension $n$ of the data}: As $n$ increases, we have an exponential decrease in the probability that two points from the two Gaussian clouds have distance smaller than $\epsilon n$. Thus, we expect our algorithm to work better for high dimensional data.
\item \emph{The mean squared distance between the clusters $\mu^2 n$}: As $\mu^2$ increases (i.e., as the data are better separated), the probability that they have small distance in the presence of missing data gets smaller.  This also allows for a wider range of $\epsilon$ and $\gamma$. 
\item \emph{The probability of a coordinate being present $p$}: First, we note that as $p$ increases, $q$ increases. Then, as $q$ increases (i.e., we have more data present), the probability that two points from the two Gaussian clouds have distance smaller than $\epsilon n$ decreases.  Additionally, for all $q > 0$, we have a feasible range for $\epsilon$. Thus, for any percentage of missing data, if we have enough data points, we can use {\sc MR-missing} for dimensionality reduction.
\end{itemize}

The final parameter $\gamma$ has a range of values it can take on and we could optimize over it to get the smallest possible bound.

Finally, as noted before, our method of estimating distances only decreases distances. Thus, points that were close together stay close together. Now by the above the theorem we see that if points were initially far apart, then our method of estimating distances keeps them far apart with high probability. 

Then, since our method of repairing the metric only increases the distance, we see that we maintain the local neighborhood structure with high probability, in this data model.

\section{Experiments}
\label{sec:experiments}

Let us now verify that our algorithm does will in practice through a variety of experiments. 
Dimensionality reduction and clustering algorithms are normally used when we have unlabeled date (i.e., in the regime of unsupervised learning). In this case figuring out the ground truth can be difficult. Hence, there are no natural numerical metrics on unlabeled data that we can readily use to evaluate our algorithm. Hence, decided to test the effectiveness of our algorithm on both unlabeled and labeled data in a variety of different experiments. 

\subsection{Unlabeled Data}

For unlabeled data, we tested the performance of {\sc MR-Missing} on synthetic manifolds visually as well as compared our algorithms against nlPCA and mDRUR numerically. 

\subsubsection{Synthetic Manifolds}

Let us first define the six manifolds we tested our algorithm on. For notational convenience, let $U(n,m)$ be an $n \times m$ matrix with entries drawn uniformly at random from $[0,1]$ and let $N(n,m)$ be an $n \times m$ matrix with entries are drawn from a standard Gaussian distribution. Finally, we shall write $f.(X)$ to represent applying $f$ coordinate wise to all elements of $X$. We generated the six synthetic manifolds as follows:

\begin{enumerate}
\item $M_1 = \cos.(U(2000,2)\cdot N(2,30))$
\item $M_2 = \cos.(\text{sigmoid}.(U(2000,2)\cdot N(2,30)) \cdot N(30,300))$
\item $M_3$ is a three dimensional manifold where $x,y$ are drawn from a standard normal and $z = e^{-\sqrt{x^2+y^2}}$
\item $M_4$ is a three dimensional manifold where $x,y$ are drawn from a uniform on $[0,1]^2$ and $z = 20e^{(-(x^2+y^2))}$
\item $M_5 = \cos.(M_{\text{Swiss Roll}})$. Where $M_{\text{Swiss Roll}}$ is the manifold used in the introduction.
\item $M_6$ is closed helical curve in three dimensions. Where $u$ is uniformly drawn from $[0,4\pi]$, and $v = 0.5u$. Then $x = (3 + \cos(u))\cos(v)$, $y = (3 + \cos(u))\sin(v)$, and $z = \sin(u)$
\end{enumerate}

In each case we start with 2000 data points on the high dimensional representations of the six manifolds $M_1, \hdots, M_6$. We then compute low dimensional representations using the full data set and using data set with missing entries as follows. We ran {\sc Isomap} with the true distance matrices to get the low dimensional projections, as depicted on the left hand side of Figures 3 and 4. We then picked 40\% on the entries uniformly at random, declared these entries to be missing, and used {\sc MR-Missing} to get the projections depicted on right of Figures 3 and 4. We used the same {\sc Isomap} parameters across both algorithms. We then compared how there projections looked visually. In most cases our algorithm did well in preserving the general structure of the low dimensional projection as can be seen in Figures 3 and 4.

\begin{figure}[h!]
\centering

\subfloat[$M_1$ with full data]{\includegraphics[width = 0.49\linewidth]{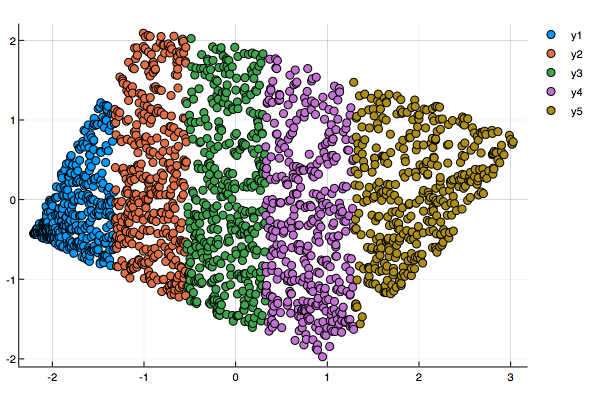}}
\subfloat[$M_1$ with missing data]{\includegraphics[width = 0.49\linewidth]{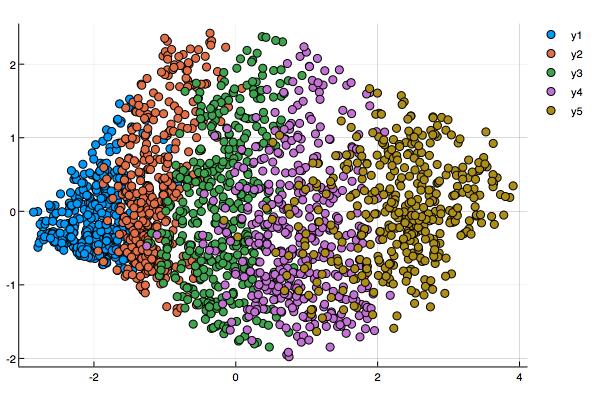}}

\subfloat[$M_2$ with full data]{\includegraphics[width = 0.49\linewidth]{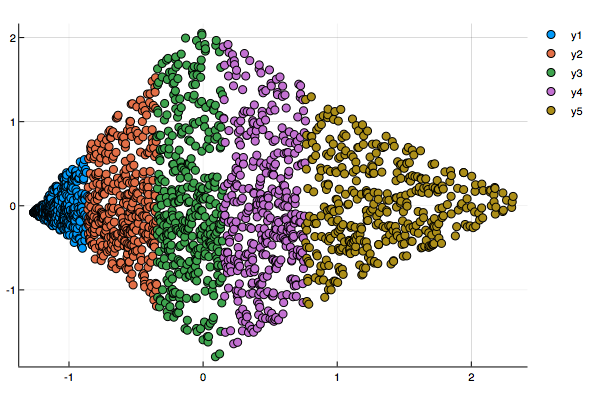}}
\subfloat[$M_2$ with missing data]{\includegraphics[width = 0.49\linewidth]{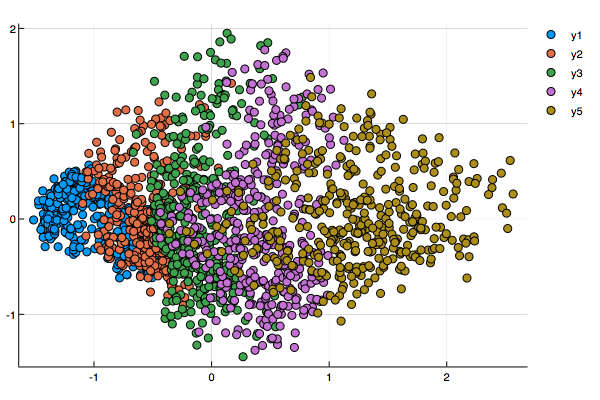}}

\subfloat[$M_3$ with full data]{\includegraphics[width = 0.49\linewidth]{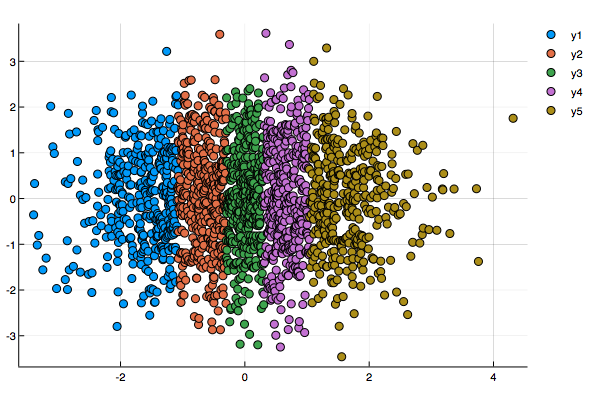}}
\subfloat[$M_3$ with missing data]{\includegraphics[width = 0.49\linewidth]{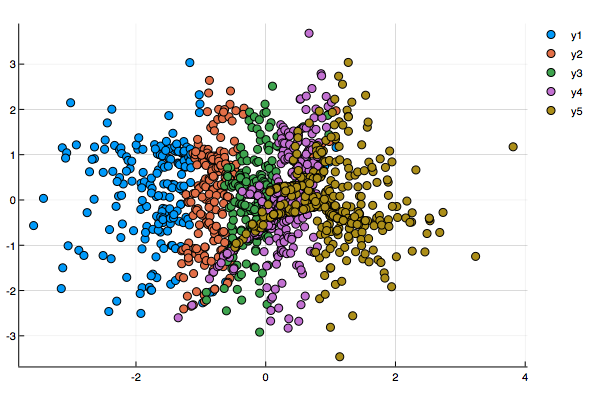}}

\caption{The two-dimensional embeddings produced by {\sc Isomap} with complete data (left) versus the two-dimensional embedding produced by {\sc Isomap} where 40\% of the data is missing and we use {\sc MR-Missing} to correct the distance matrix for the manifolds $M_1,M_2,M_3$ (right).}
\end{figure}

\begin{figure}[h!]
\centering

\subfloat[$M_4$ with full data]{\includegraphics[width = 0.49\linewidth]{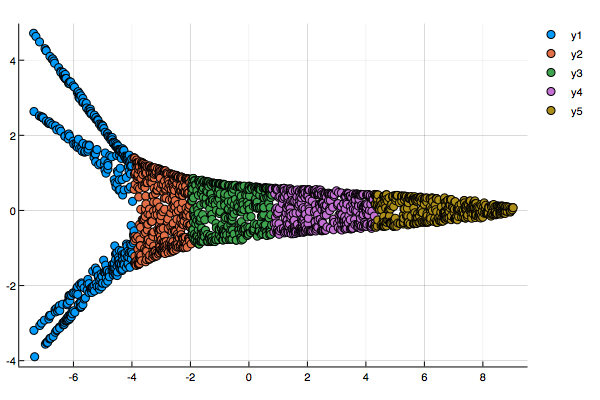}}
\subfloat[$M_4$ with missing data]{\includegraphics[width = 0.49\linewidth]{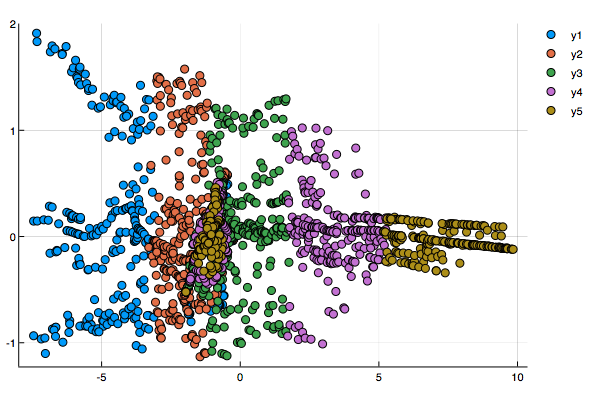}}

\subfloat[$M_5$ with full Data]{\includegraphics[width = 0.49\linewidth]{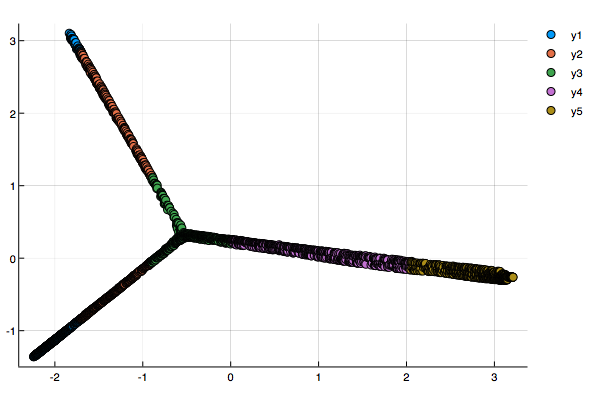}}
\subfloat[$M_5$ with missing Data]{\includegraphics[width = 0.49\linewidth]{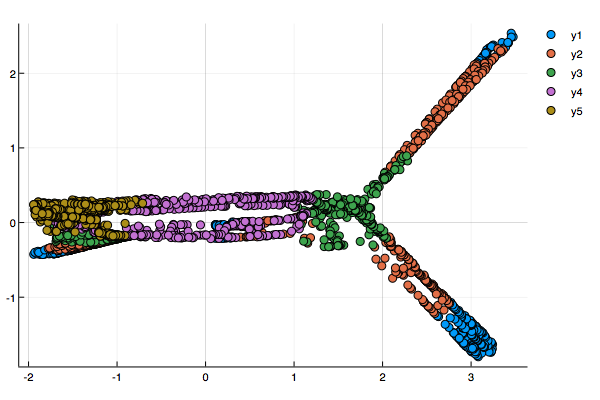}}

\subfloat[$M_6$ with full data]{\includegraphics[width = 0.49\linewidth]{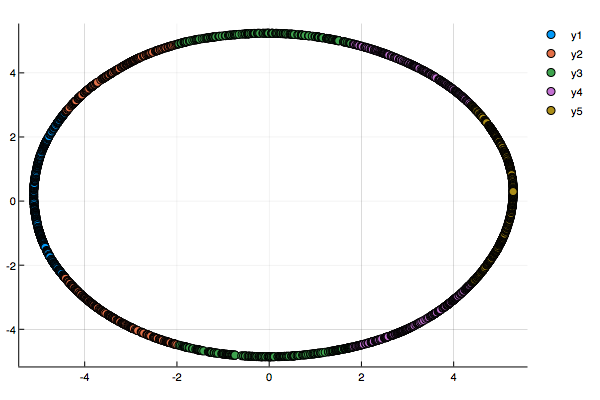}}
\subfloat[$M_6$ with missing data]{\includegraphics[width = 0.49\linewidth]{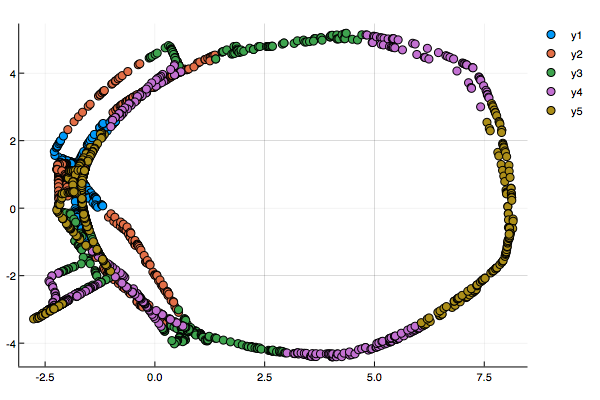}}

\caption{The two-dimensional embeddings produced by {\sc Isomap} with complete data (left) versus the two-dimensional embedding produced by {\sc Isomap} where 40\% of the data is missing and we use {\sc MR-Missing} to correct the distance matrix for the manifolds $M_4,M_5,M_6$ (right).}
\end{figure}

In each case, we can see that {\sc MR-Missing} does well at preserving not only the general shape of the low dimensional, but also in maintaining the relative ordering of the data points with some minor distortion. Therefore, it is also useful for applications that use these projections for clustering and classification. We test this in the next.

\subsubsection{MR-missing vs nlPCA vs mDRUR}

We compare {\sc MR-Missing} against Non Linear PCA (nlPCA) and Missing Data Recovery Through Unsupervised Regression (mDRUR) on the MNIST data set.  nlPCA and mDRUR are both methods to complete missing data on a manifold, whereas our algorithm is a method to estimate and correct distances so that we can use dimensionality reduction algorithms on data sets with missing entries. So, we must compare all of these algorithms on low dimensional representations. To that end, for the nlPCA and mDRUR algorithm, we first filled in missing data in the MNIST data set and then used {\sc Isomap} to determine the low dimensional projection. To quantify how well these representations do we compared these representations to the low dimensional representation computed by using {\sc Isomap} with the complete data set. We kept the parameters for {\sc Isomap} fixed across all four projections. 

We compared these projections in the following manner. Let $P$ be the ground truth projection and $\hat{P}$ the projection computed by any algorithm on the data set with missing entries. We first aligned the two projections, using the Procrustes method because the representation returned by Multidimensional Scaling is unique only up to rotation and translation. Additionally, since we care only about the relative positions of the points and not the magnitude of the distances between them, we also allowed for scaling. That is, to ascertain the quality of the projection $\hat P$, we want to find 
\[ 
	{\rm argmin}_{Q,\alpha, \mu} \|X - \alpha YQ - 1_m \mu^t \|_F 
\] 
where $X$ is the data matrix that we are trying to align $Y$ with. Here, each row is a data point,  $\alpha$ is the scaling constant, $Q$ is our orthogonal rotation matrix and $\mu$ is our translational vector and $1_m$ is the vector of all ones in $m$ dimensions. 

This problem has a closed form solution in terms of the SVD decomposition of $X$. Finally, if we let $\tilde{P}$ be our new scaled, rotated, and translated projection that best aligns with the ground truth projection $P$, then we calculate the relative error between the projections as 
\[ 
	\frac{\|P-\tilde{P}\|_F}{\|P\|_F}. 
\]

The data set we used is the first 1000 images of the digit 0,1,2,3,4 from MNIST. For nlPCA we used a $784 \times 800 \times 12 \times 800 \times 784$ structure. For mDRUR we initially filled in the matrix using the singular value projection method. We then calculated an initial 12 dimensional low representation using {\sc Laplacian Eigenmaps}.  We can see the relative errors in Table 1. 

\begin{table*}[h!]
  \centering
 \begin{tabular}{|c|c||c|c|c|c|c|c|c|c|} \hline
 Algorithm & \% Missing & 2D & 3D & 4D & 10D & 12D & 20D & 50D & 100D\\ \hline \hline
 nlPCA & 40 &  0.363 & 0.350 & 0.385 & 0.404 & 0.451 & 0.514 & 0.623 & 0.686 \\ \hline 
 mDRUR & 40 & 0.369 & 0.363 & 0.359 & 0.420 & 0.427 & 0.505 & 0.630 & 0.717 \\ \hline
 \textbf{MR-Missing} & 40 & \textbf{0.291} & \textbf{0.274} & \textbf{0.263} & \textbf{0.339} & \textbf{0.359} & \textbf{0.438} & \textbf{0.572} & \textbf{0.658} \\ \hline \hline
 nlPCA  & 50 & \textbf{0.324}  & 0.330 & \textbf{0.317} & 0.394  &  0.441 &  0.506 & 0.621  & \textbf{0.685}  \\ \hline
 mDRUR & 50 & 0.497 & 0.505 & 0.471 & 0.518 & 0.542 & 0.587 & 0.707 & 0.777 \\ \hline
 \textbf{MR-Missing} & 50 & \textbf{0.323}  & \textbf{0.317}  &  0.328 & \textbf{0.393} & \textbf{0.417} & \textbf{0.482}  & \textbf{0.615}  &  0.707 \\ \hline \hline
  \textbf{nlPCA}  & 60 &  \textbf{0.366} &  \textbf{0.365} & 0.399 & \textbf{0.405} & \textbf{0.441} &  0.520 & \textbf{0.635} &  \textbf{0.696}\\ \hline
  mDRUR & 60 & 0.595 & 0.595 & 0.573 & 0.654 & 0.667 & 0.712 & 0.802 & 0.849 \\ \hline
  \textbf{MR-Missing} & 60 & \textbf{0.369} & \textbf{0.370}  & \textbf{0.376} & 0.436 & \textbf{0.448}  &  \textbf{0.505} &  0.653 & 0.741 \\ \hline \hline
 \textbf{nlPCA}  & 70 & \textbf{0.373} & \textbf{0.373} & \textbf{0.391}& \textbf{0.432} & \textbf{0.465} & \textbf{0.533} & \textbf{0.643} & \textbf{0.706} \\ \hline
  mDRUR & 70 & 0.924 & 0.874 & 0.820 & 0.825 & 0.830 & 0.854 & 0.898 & 0.920 \\ \hline
  MR-Missing & 70 & 0.484 & 0.491 & 0.595 & 0.498 & 0.510 & 0.573 & 0.697  & 0.784 \\ \hline 
  \end{tabular} 
  \caption{Table comparing the relative error of the projection of MNIST data obtained via NLPCA vs mDRUR vs  MR-missing for various different dimensions and percentage of data missing}
\end{table*}

We can see that in all cases mDRUR does the worst so we will focus on comparing {\sc MR-Missing} versus nlPCA. When we have 40\% missing data {\sc MR-Missing} does better than the nlPCA version in all cases. For 50\% missing we see that nlPCA does better or as well in some cases with {\sc MR-Missing} still doing better in a majority of the cases. For 60\% missing the algorithms have similar results and nlPCA does better in the case of 70\% missing data. Thus, {\sc MR-Missing} does better for smaller percentage of missing data while for higher percentage of missing data nlPCA does better. Additionally we can see that as the dimension increases both methods have worse errors. We posit that this occurs because to compute a rank $k$ projection $P (\hat{P}, \tilde{P})$ we are using the first $k$ singular values of the distance matrix computed by {\sc Isomap} (see section 1B). Hence our estimation the distance matrix does better at preserving the larger singular values as compared to the smaller values. 

Let us also take a closer look at what our algorithm does in the case of 70\% missing data. See the two dimensional representations shown in Figure 5.
\begin{figure}[!h]
\centering

\subfloat[ISOMAP with actual distances]{\includegraphics[width = 0.49\linewidth]{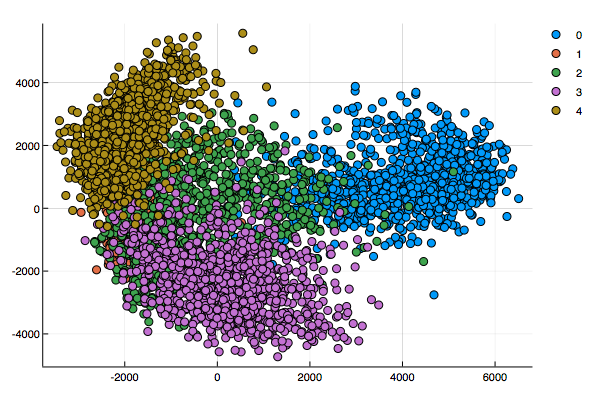}}
\subfloat[ISOMAP with repaired distances]{\includegraphics[width = 0.49\linewidth]{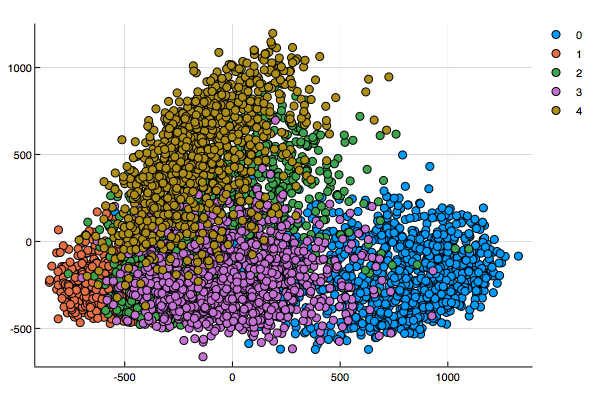}}

\caption{Two-dimensional projections of the first 1000 images of the digits 0,1,2,3,4 from MNIST using {\sc Isomap} with true distance and {\sc Isomap} with distance obtained from {\sc MR-Missing} when 70\% of the data is missing.}
\end{figure}

We see that we have a different looking projection, but the projection still does well in maintaining the clustering of the data, as we predicted theoretically in Section 3.  In the next subsection, we test the effectiveness of these computed low dimensional representations for classification.

\subsection{Labeled Data}

One of the main reasons to find a lower dimensional data representation is to efficiently carry out standard machine learning tasks, such as classification. In this subsection, we test the usefulness of the low dimensional representations produced by {\sc MR-Missing} for classification. We calculate a 100-dimensional representation of 500 images of each digit from MNIST. This is our training set. We then used the equation from \cite{Bengio2003} to calculate the projections for an additional 100 images of each digit. This is our test set. Then, for classification we trained an SVM (Kernel: RBF, $C = 200$, $\gamma = 0.0000002$) for classification. We then obtained the following classification accuracy for various amounts of missing data shown in Table 2. We also ran this classifier with no missing data as a benchmark. As we expect, with missing data, we do not have as high accuracy as we do with complete data, but from our experiments we see that even in the presence of missing data, {\sc Isomap} with {\sc MR-Missing} produces embeddings on which we still have reasonably high accuracy. Particularly in the cases of 40\% and 50\% missing data, we have accuracy of over 90\% and the accuracy does not drop off drastically until we get to 80\% missing data. 

\begin{table}[h]
\centering
\begin{tabular}{|c|c|c|c|c|c|c|c|} \hline
\% missing & 0 & 40 & 50 & 60 & 70 & 80 & 90 \\ \hline
Accuracy & 0.94 & 0.91  & 0.90 & 0.86 & 0.77 & 0.20 & 0.10 \\ \hline
\end{tabular}
\caption{Table showing the accuracy of an SVM trained on the low dimensional projections produced by MR}
\end{table}

It is important to note that both the test and the training set had data points missing. We see that {\sc MR-Missing} does a good job of maintaining the original clusters of the data. 
\section{Conclusion and Future Work}

As we can see {\sc MR-Missing} has excellent experimental results. That is, {\sc MR-Missing} is a method by which we can use traditional dimensionality reduction algorithms in the presence of missing data. While we have some theoretical justification for this observed performance, more work needs to be done in exploring the effectiveness of our method of estimating distances in more general scenarios. 

Additionally, we did not consider in a detailed fashion other corrupted data or distance models. One approach to metric completion (model scenario 4) is to take the given distances, treat these as edges on a graph, and run APSP on this graph to fill in the missing distances. It is possible that APSP modifies some of the given distance information while also filling in the missing values. Thus, a natural question is are there conditions on the given data that guarantee that an APSP algorithm will not change the given data while simultaneously repairing those that are missing? Gilbert and Sonthalia~\cite{GilbertSonthalia:GraphMetric2018} provide some analysis which suggests a more general version of metric repair may be applied to the metric problem.

\begin{thm}(Gilbert and Sonthalia~\cite{GilbertSonthalia:GraphMetric2018}) 
Suppose $G$ is a weighted chordal graph such that no three cycle is broken. Then if we run APSP on this graph, the shortest path between any two adjacent vertices, is the edge connecting them. 
\end{thm}

\begin{cor} (Gilbert and Sonthalia~\cite{GilbertSonthalia:GraphMetric2018})
If the given distances form a graph $G$, where $G$ is a weighted chordal graph such that no 3 cycle is broken, then this partial distance information can be completed into a metric. 
\end{cor}

The condition that the given data satisfy a chordal graph appears in Positive Semi-Definite matrix completion and Euclidean distance matrix completion as well. While the first theorem tells us when we can use APSP to complete a metric, it doesn't tell us what properties this new metric satisfies. Hence, the problem of completing a distance matrix for a general metric warrants further investigation. 


In the model scenario 3,  when we have a corrupted distance matrix, we may not always want to increase distances. In some cases we might want to decrease distances. Hence, we would need a general metric repair algorithm. Fan, et al.~\cite{Raichel2018} present an algorithm that runs in $\theta(n^6)$ and Gilbert and Sonthalia~\cite{GilbertSonthalia:GraphMetric2018} present an alternative algorithm that runs in $O(n^5)$. Both of these algorithms are impractical and cannot be used on large data sets. Developing faster algorithms for general metric repair and ascertaining the usefulness of such methods for corrupted distance matrices are two avenues for future work.

\nocite{*}
\bibliography{IEEEabrv,citations} 
\bibliographystyle{IEEETran}


\end{document}